\documentclass{article}
\usepackage{lipsum}

\usepackage[preprint]{neurips_2021}

\usepackage[acronym]{glossaries}

\makeglossaries

\newglossaryentry{latex}
{
    name=latex,
    description={Is a mark up language specially suited 
    for scientific documents}
}

\newglossaryentry{maths}
{
    name=mathematics,
    description={Mathematics is what mathematicians do}
}

\newglossaryentry{isomorphic}
{
    name=isomorphic,
    description={a}
}

\newglossaryentry{faithful}
{
    name=faithful,
    description={}
}

\newglossaryentry{group}
{
    name=group,
    description={}
}

\newglossaryentry{subdir}
{
    name=sub-direct product,
    description={pairing}
}

\newglossaryentry{relation}
{
    name=relation,
    description={A relation between two sets is a collection of ordered pairs containing one object from each set}
}

\newglossaryentry{homogeneous space}
{
    name=homogeneous space,
    description={If the action of $G$ on $\mathcal{X}$ is transitive, we say that $X$ is a homogeneous space of $G$}
}

\newglossaryentry{transitive}
{
    name=transitive,
    description={Transitivity is the property that taking any $x_0 \in \mathcal{X}$, any other $x\in \mathcal{X}$ can be reached by the action of some $g \in G$, i.e., $x= g(x_0)$}
}

\newglossaryentry{corollary}
{
    name=corollary,
    description={a less important theorem in which the (usually short) proof relies heavily on a given more significant theorem. It is often stated that: “this is a corollary of Theorem X”}
}

\newglossaryentry{regularity}
{
    name=regularity,
    description={Look at definition of semi-regular(free) and regular}
}

\newglossaryentry{semi-regular}
{
    name=semi-regular,
    description={A group action of a group on set is called semi-regular or free if for any two elements in the set, there is at most one element of the group that transfers the first element to the second.}
}

\newglossaryentry{regular}
{
    name=regular,
    description={A group action is regular if and only if it is both transitive and free}
}

\usepackage[utf8]{inputenc} 
\usepackage[T1]{fontenc}    
\usepackage{hyperref}       
\usepackage{url}            
\usepackage{booktabs}       
\usepackage{amsfonts}       
\usepackage{amsmath, amssymb, amsthm} 
\usepackage{microtype}      
\usepackage{xcolor}         
\setcitestyle{authoryear,open={(},close={)}} 

\usepackage{nicefrac}       
\usepackage{microtype}      
\usepackage{tcolorbox}
\usepackage{wrapfig,subfig}
\usepackage{graphicx}
\usepackage{soul}
\usepackage{booktabs}
\usepackage{algorithmic}
\usepackage{stackengine}
\usepackage[ruled,vlined]{algorithm2e}
\usepackage{mathtools}

\usepackage{diagbox}
\usepackage{tikz}
\usetikzlibrary{positioning}

\newtheorem{theorem}{Theorem}
\newtheorem{definition}[theorem]{Definition}

\title{Circular-Symmetric Correlation Layer based on FFT}

\author{%
  Bahar Azari \\
  {\footnotesize Department of Electrical \& Computer Engineering}\\
  Northeastern University, USA\\
  Boston, MA 02115 \\
  \texttt{azari.b@northeastern.edu} \\
  \And
  Deniz Erdo\u{g}mu\c{s} \\
  {\footnotesize Department of Electrical \& Computer Engineering}\\
  Northeastern University, USA\\
  Boston, MA 02115 \\
  \texttt{Erdogmus@ece.neu.edu} \\
}

\begin{document}

\maketitle

\begin{abstract}
Despite the vast success of standard planar convolutional neural networks, they are not the most efficient choice for analyzing signals that lie on an arbitrarily curved manifold, such as a cylinder. The problem arises when one performs a planar projection of these signals and inevitably causes them to be distorted or broken where there is valuable information. We propose a Circular-symmetric Correlation Layer (CCL) based on the formalism of roto-translation equivariant correlation on the continuous group $S^1 \times \mathbb{R}$, and implement it efficiently using the well-known Fast Fourier Transform (FFT) algorithm. We showcase the performance analysis of a general network equipped with CCL on various recognition and classification tasks and datasets. The PyTorch package implementation of CCL is provided online. 
\end{abstract}

\section{Introduction}
Planar convolutional neural networks, widely known as CNNs, which have been exceptionally successful in many computer vision and machine learning tasks, such as object detection, tracking, and classification, are characterized by pattern-matching filters that can identify motifs in the signal residing on a 2D plane. However, there exists various applications in which we have signal lying on a curved plane, e.g., temperature and climate data on the surface of the (spherical) earth, and $360^\circ-$panoramic images and videos from surveillance cameras to name a few.  
\begin{wrapfigure}[20]{r}{0.45\textwidth}
  \begin{center}
  \includegraphics[width=0.45\textwidth]{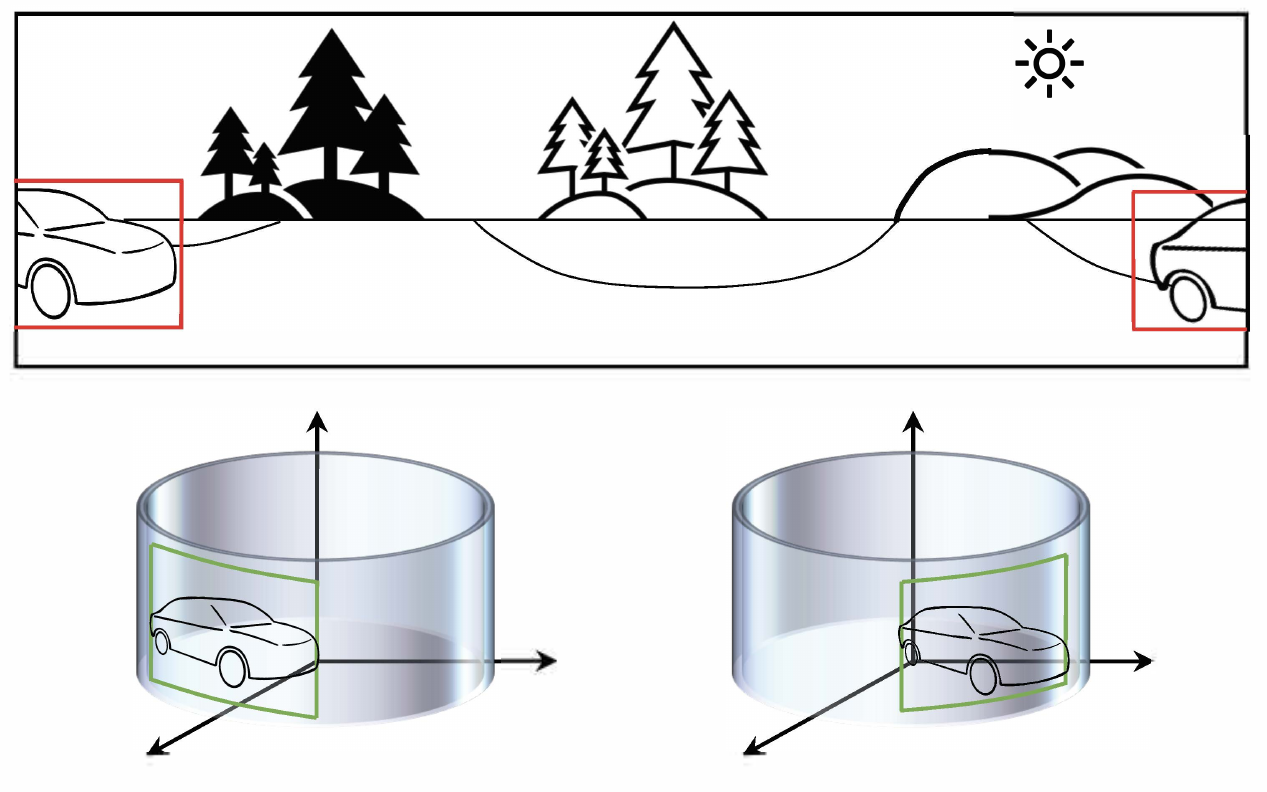}
  \end{center}
  \caption{Object breakage in $360^\circ$-panoramic image unwrapping. \textbf{Top:} The car has been subjected to image cut. \textbf{Bottom:} Cognition tasks should be invariant to shifting of the object on the surface the cylinder.}
  \label{fig:cyCut}
\end{wrapfigure}
Analyzing signals in the above-mentioned applications is achievable by using the planar projection of the signals. In the case of the $360^\circ-$panoramic data, which is the interest of this study, the images are usually unwrapped to a standard 2D image to be treated as input feature map. However, the resulting arbitrary breakage of the signal may be destructive in object-detection tasks (see figure.~\ref{fig:cyCut}). Furthermore, a commonly negligible drawback of CNNs becomes important in this scenario. A convolution kernel produces a single value information for a region in the image that is covered with the kernel at specific shift. The information at the border are however neglected since the kernel shift needs to stop at a margin equal to half the size of the kernel. This is detrimental for panoramic image processing because potential valuable information in the border of the image (The car in figure.~\ref{fig:cyCut}). Even the zero-padding techniques, applied to the out-of-image regions, introduce increasing distortion growing from the border towards the interior as we go deeper in the CNN. There are other proxy techniques that try to alleviate the border information loss problem, such as input
padding (see \cite{shi2015deeppano}), or creating more training data using multiple circular shifted versions of the original data (see \cite{lo2002multiple}). The former increases computational time and memory consumption and the latter increases the training time and could have adversary effect on the kernels' representational capacity as it exposes them to more and more border area.

Another limitation of CNNs that becomes noticeable in the case of analyzing panoramic data is related to what is known as \textit{invariance} and \textit{equivariance} property of neural network as a function. For defining these properties, we consider a family, or a ``group'', of transformations (e.g., rotations, or translations) of input and output to a given layer of the neural network. The elements of the group can ``act'' on the input and output of each layer in some specific way. The neural network is invariant to the action of the group, if transformations of the input does not change the output. Otherwise, it is equivariant if, as we transform the input, the output is transformed according to another certain action of the group. The convolution layers are empirically known to be invariant to small translations of their input image, but they are not completely immune to relatively large rotations (\cite{goodfellow2009measuring, schmidt2012learning, he2015spatial, lenc2015understanding, jaderberg2015spatial, cohen2016group, dieleman2016exploiting}). Hence, they may fail on the tasks that requires invariance to a specific transformation, and and on the data that includes a wide range of that transformation. Figure.~\ref{fig:cyCut}(Bottom) shows an example of this phenomenon. The object (car) identification task should be invariant to rotation around the $z$-axis.

\par 
Nevertheless, the building block of CNN, i.e. convolution or cross-correlation, has the potential equivariance property that can be exploited to construct a network suitable for (horizontal) translation-invariant task, such as object detection (figure.~\ref{fig:cyCut}). Therefore, for a systematic treatment of analyzing the $360^\circ-$panoramic data, we propose a circular-symmetric correlation Layer (CCL) based on the formalism of roto-translation equivariant correlation on the continuous group $S^1 \times \mathbb{R}$ -- a group constructed of the unit circle and a the real line.

\par
We implement this layer efficiently using the well-known Fast Fourier Transform (FFT)  and discrete cosine transform (DCT) algorithm. We discuss how the FFT yields the exact calculation of the correlation along the panoramic direction due to the circular symmetry and guarantees the invariance with respect to circular shift. The DCT provides a improved approximation with respect to transnational symmetry compared to what we observe in CNNs.    

\par
We showcase the performance analysis of a general network equipped with CCL on various recognition and classification tasks and datasets. The PyTorch package implementation of CCL is provided online.

Our contributions are as follows:
\begin{enumerate}
    \item Theoretical definition of circular-symmetric correlation on the surface of a cylinder. 
    \item Efficient implementation of CCL based on FFT and DCT.
    \item Experimental analysis that shows the outperformance of neural networks equipped with CCL. 
\end{enumerate}

\section{Related Work}
The outstanding ability of CNN in processing spatially and temporally correlated signals comes from the fact that it exploits the transnational symmetry and the equivariance property of its correlation layers. In fact, correlation layers in CNN are equivariant to moderate translations. In other words, a trained pattern should be able to detect a particular signal at a specific location in the image, independent of this location. Due to this powerful property, recently, there has been increasing attempt to generalize the idea of CNN to other spaces, and symmetry group \citep{gens2014deep, olah2014groups,dieleman2015rotation,guttenberg2016permutation,dieleman2016exploiting,cohen2016steerable, ravanbakhsh2016deep,zaheer2017deep,ravanbakhsh2017equivariance,worrall2017harmonic,maron2020learning,dym2020universality}.    

Most of these studies focus on discrete groups. For example, the investigation of discrete $90^\circ$ rotations acting on planar images in the work of  \cite{cohen2016group}, permutations of nodes in graphs in \cite{maron2018invariant}, or permutations of points in the point cloud in \cite{zaheer2017deep}. Recent works \citep{cohen2018spherical,cohen2019general} have been investigating equivariance to continuous groups and generalized the CNN to various spaces. \citep{kondor2018generalization, cohen2018spherical} use the generalized Fourier transform for group correlation and provided a formalism to efficintly implement these layers. Circular symmetry, which is the interest of this paper, has also been empirically studied in \citep{schubert2019circular,papadakis2010panorama,kim2020cycnn}, but non of these works addressed the issue in a formal analytic way.

\section{Circular-Symmetric Correlation Layer}
In order to learn a function that predicts a quantity based on a spatially-correlated signal, such as an image, we need to perform cross-correlation (correlation, in short), namely, we slide a kernel (filter) \textit{throughout} the signal and measure their similarity. We have the familiar case of a classical planar $\mathbb{R}^2$ correlation, in which the output value at translation $x\in\mathbb{R}^2$ is computed as an inner product between the input and a kernel, translated to $x$. However, correlation is not limited to the signals on $\mathbb{R}^2$, and in our case we are interested in images on the surface of a cylinder. We begin our discussion by introducing the correlation on the surface of a cylinder. To do so, we first start with defining the mathematical building blocks.
\subsection{Preliminaries and Notation}
\paragraph{Cylinder} We consider the lateral surface of a cylinder, a manifold, which is constructed by the combination of two other manifolds -- a circle and a line segment\footnote{It is either an infinite line, or a line segment without its end points which is also a manifold.}. The unit circle $S^1$, defined as the set of points $z \in\mathbb{R}^2$ with norm $1$, is a one-dimensional manifold that can be parameterized by polar coordinate $\varphi \in \left[0,2\pi\right]$. Cartesian product of $S^1$ with a line $\mathbb{R}$ (or, a line segment $(-a,a)$) constructs a two-dimensional manifold, known as a cylinder $\mathbb{X} = S^1 \times \mathbb{R}$ (or, $S^1 \times (-a,a)$ in case of having a line segment). We characterize the set of points on the lateral surface of the cylinder by cylindrical coordinates $\varphi \in \left[0,2\pi\right]$ and $z\in\mathbb{R}$, and define circular-symmetric signals and convolution kernels as continuous functions on this surface $f:\mathbb{X}\mapsto\mathbb{R}^K$, where $K$ is the number of channels.

\paragraph{Rotation and Translation on Cylinder surface} The set of rotations around and translations along the $z$-axis is a subgroup of $\text{SE}(3)$, the ``special Euclidean group'', denoted as $\mathcal{G}\leq \text{SE}(3)$ and is isomorphic to $\mathbb{X}$, i.e., $\mathcal{G}=S^1 \times \mathbb{R}$. The action of an element $\xi$ in $\mathcal{G}$ is a pair $(R_{\psi}, \nu)$, where $R_{\psi}$ belongs to a subgroup of the ``special orthogonal group'' $\text{SO}(3)$ representing a rotation by $\psi$ around $z$-axis, and a translation by $\nu\in\mathbb{R}$ along $z$-axis. The representation of $\mathcal{G}$ corresponds to the set of all $ 4 \times 4$ transformation matrices of the form
\begin{equation}
    \mathcal{G}=\left\{
    {\scriptsize
    \begin{pmatrix}
         &  &  & 0\\
         \multicolumn{3}{c}
        {\raisebox{\dimexpr\normalbaselineskip-.8\ht\strutbox-.5\height}[0pt][0pt]
        {\scalebox{2}{$R_{\psi}$}}} & 0 \\
         &  &  & \scalebox{1.3}{$\nu$} \\
        0 & 0 & 0 & 1
    \end{pmatrix}
    } 
    \middle| \psi \in [0,2\pi]~\text{and}~ \nu \in \mathbb{R} \right\}, 
    \label{eq:SEgroup}
\end{equation}
where $R_{\psi}$ is a 3D rotation matrix. The specific form of transforming filters and functions on the cylindrical surface, which we consider in this study, corresponds to applying the roto-translation operator $L_{\xi}$ which takes a function $f:\mathbb{X}\mapsto\mathbb{R}^K$ and produces a shifted version by rotating it around and translating it along the $z$-axis as (see figure.~\ref{fig:cyCut}, bottom: rotation of the car around the $z$-axis.):

\begin{equation}
    [L_{\xi}f](x) = f(\xi^{-1}x).  
    \label{eq:operator}
\end{equation}

As we explained earlier, since $\mathcal{G}$ is a group and groups contain inverses, for $\xi, \xi' \in \mathcal{G}$ we have $L_{\xi\xi'} = L_{\xi}L_{\xi'}$. We can show this using inverse and associative property of groups as:
\begin{equation}
    [L_{\xi\xi'}f]\left(x\right) = 
    f\left((\xi\xi')^{-1}x\right) = 
    f\left(\xi'^{-1}(\xi^{-1}x)\right) = 
    [L_{\xi'}f]\left(\xi^{-1}x\right) = 
    [L_{\xi}L_{\xi'}f]\left(x\right).
    \label{eq:prop}
\end{equation}

\subsection{Correlation on Cylinder}
To define the correlation we begin with the familiar definition of inner product. We define the inner product on the vector space of cylindrical signals as:
\begin{equation}
    \langle f,h\rangle = \int_{\mathbb{X}} \sum_{k=1}^{K} f_k(x)\,h_k(x)dx,
    \label{eq:inner}
\end{equation}

where the integration measure $dx$ denotes the Haar measure (invariant integration measure) on the lateral surface of the cylinder and it is equal to  $d\varphi dz$ in cylindrical coordinate. Due to the invariance of the measure, the value of the integral of a function affected by any $\xi \in \mathcal{G}$ remains the same, namely, $\int_{\mathbb{X}}f(\xi x)dx = \int_{\mathbb{X}}f(x)dx$ for all $\xi\in \mathcal{G}$. Using the inner product in (\ref{eq:inner}), we define the correlation of signals and filters on the surface of the cylinder. Given a point on the cylinder surface $x \in \mathbb{X}$, a transformation on the subgroup of $\text{SE}(3)$,  $\xi \in \mathcal{G}$, and functions $f(x)$ and $h(x)$ the correlation is defined as:

\begin{equation}
    [f \star h](\xi) = \langle L_{\xi}f,h\rangle = \int_{\mathbb{X}} \sum_{k=1}^{K} f_k(\xi^{-1}x)h_k(x)dx.
    \label{eq:corr}
\end{equation}

Note that the correlation in (\ref{eq:corr}) is also equivalent to $\langle f,L_{\xi^{-1}}h\rangle$ as the value of the correlation at a shift $\xi$ is equal to the inner product of $f$ and $h$, where either $f$ is shifted by $\xi$, or $h$ is shifted by the inverse of $\xi$ ( $\xi^{-1}$). Therefore, if we express the point $x$ as $x = (\varphi, z)$, the transformation as  $\xi = (\psi, \nu)$, and the Haar measure as $dx = d\varphi dz$, the correlation in (\ref{eq:corr}) can be rewritten as:
\begin{equation}
    [f \star h](\xi) = \langle L_{\xi}f,h\rangle = \int_{\mathbb{R}}\int_{0}^{2\pi} \sum_{k=1}^{K} f_k(\varphi-\psi,z-\nu)\,h_k(\varphi-\psi,z-\nu)d\varphi dz,
    \label{eq:corr2}
\end{equation}

where the integral with respect to $\varphi$ is the circular cross-correlation. It is worthwhile to mention that the resulting correlation function lies on the group $\mathcal{G}$ which is isomorphic to the space $\mathbb{X}$ that the initial functions have lied on, namely $S^1 \times \mathbb{R}$.   

\subsection{Equivariance of Correlation Layers}
For the correlation in (\ref{eq:corr2}), which is defined in terms of operator $L_\xi$, i.e., rotation $\varphi$ around and translation $\nu$ along $z$-axis, we can show the important equivariance property known for all convolution and correlation layers. We express mathematically what we informally stated earlier. 
\par\textbf{Group actions:} For a set of points $\mathbb{X}$, we have a group $\mathcal{G}$ that acts on $\mathbb{X}$. This means that for each element $\xi \in \mathcal{G}$, there exist a transformation $T_{\xi}:\mathbb{X}\rightarrow\mathbb{X}$ corresponding to group action $x \mapsto T_{\xi}(x)$. We showed this simply as $\xi x$ to simplify notation. As we have seen earlier, the action of $\mathcal{G}$ on $\mathbb{X}$ extends to functions on $\mathbb{X}$ (induced action) and that is what we have denoted as the operator  $L_{\xi}: f \mapsto f'$ which is $f'(x) = [L_{\xi}f](x) = f(\xi^{-1}x)$.

\textbf{Equivariance:} Equivariance is the potential property of a map between functions on a pair of spaces with respect to a group acting on these spaces through the group action.  

\begin{definition}
Let $\mathbb{X}_1$, $\mathbb{X}_2$ be two sets with group $\mathcal{G}$ acting on them. Consider $V_1$ and $V_2$ as the corresponding vector spaces of functions defined on the sets, and $L_\omega$ and $L_\omega'$ as the induced actions of $\mathcal{G}$ on these functions. We say that a map $\Phi: V_1 \rightarrow V_2$ is $\mathcal{G}$–equivariant if 
\begin{equation*}
    \Phi(L_\omega(f)) = L_\omega'(\Phi(f))\quad \forall f\in V_1,~\forall\omega \in \mathcal{G}.
\end{equation*}
\end{definition}

Considering that the map in our case corresponds to the cross-correlation function we have defined on the cylindrical surface in (\ref{eq:corr}), its equivariance with respect to the action of group $\mathcal{G} = S^1 \times \mathbb{R}$ can be demonstrated as following:

\begin{theorem}
Cross-correlation on lateral surface of a cylinder is equivariant to the  action of the group $S^1 \times \mathbb{R}$. 
\end{theorem}
\begin{proof}
Given that the group $G$ of transformations on the cylinder surface is isomorphic to the set of points on the cylindrical manifold, we have: 
\begin{align*}
       [h \star L_{\omega} f](\xi)& \stackrel{\text{by definition in (\ref{eq:corr})}}{=}\\ \langle L_{\xi}h,&L_{\omega}f\rangle 
       = \langle L_{\omega^{-1}}L_{\xi}h,f\rangle \stackrel{\text{by (\ref{eq:prop})}}{=}\langle L_{\omega^{-1}\xi}h,f\rangle = [h \star f](\omega^{-1}\xi)\stackrel{\text{by (\ref{eq:operator})}}{=} [L_{\omega} [h \star f]](\xi). 
\end{align*}

\end{proof}

where $[h \star .](\xi)$ is the cross-correlation function, and $L_{\omega}$ is a transformation operator. Note that in our case $L_{\omega} = L_{\omega}'$ since the function resulting lies on the same space as the input functions. Equivariance can be represented graphically by commutative diagram as:
\begin{center}
\begin{tikzpicture}
  \node (A) {$~~~~~~~~f~~~~~~$};
  \node (B) [below=of A] {$[h \star f](\xi)$};
  \node (C) [right=of A] {$~~~~~~~L_{\omega}f~~~~~~~~$};
  \node (D) [right=of B] {$[L_{\omega} [h \star f]](\xi)$};
  \draw[-stealth] (A)-- node[left] {\small $[h \star .](\xi)$} (B);
  \draw[-stealth] (B)-- node [below] {\small $L_{\omega}$} (D);
  \draw[-stealth] (A)-- node [above] {\small $L_{\omega}$} (C);
  \draw[-stealth] (C)-- node [right] {\small $[h \star .](\xi)$} (D);
\end{tikzpicture}
\end{center}

In the next part we explain how to implement the circular-symmetric correlation layer efficiently using the notion of Fourier transform and cross-correlation theorem.

\subsection{Implementing CCL using FFT}
Computing cross-correlation and convolution using the Fast Fourier Transform (FFT) is known to be more efficient than their direct calculation. This is an important result of the Convolution (cross-correlation) theorem, according to which, the cross-correlation between two signals is equal to the product of the Fourier transform of one signal multiplied by complex conjugate of Fourier transform of the other signal, or mathematically,  $\widehat{f \ast g} = \hat{f}\;\odot\;\hat{g}$ where $\odot$ is element-wise product. In fact, the complexity of computing the FFT is $O(n\,\text{log}\,n)$ time and the product $\odot$ has $O(n)$ complexity. On the other hand, direct spatial implementation of the correlation is of complexity $O(n^2)$. Therefore, asymptotically it is beneficial to implement the CCL using FFT. Fourier transform is  a linear projection of a function onto a set of orthogonal basis functions. For the real line ($\mathbb{R}$) and the circle ($S^1$), these basis functions are the familiar complex exponentials $\text{exp}(\imath n\theta)$, where $\imath=\sqrt{-1}$.

The input of the FFT is the spatial signal $f$ on $\mathbb{X}$, sampled on a discrete grid of the cylindrical coordinate $(\varphi,z)$. This signal is periodic in $\varphi$ and finite along $z$. Therefore, the convolution theorem holds for the dimension which is equivalent to the unwrapped version of the $\varphi$ and the usage of FFT  for implementing the correlation in this dimension is appropriate. However, we do not have the same periodicity in the dimension $z$. Hence, we use another set of basis functions (cosine waves) and as a consequence we use discrete cosine transform (DCT) in the $z$ dimension. 

\textbf{Discrete Cosine Transform (DCT):}
The DCT is the most commonly used transformation technique in signal processing and by far the most widely used linear transform in lossy data compression applications such as the JPEG image format \citep{muchahary2015simplified}.
DCT are related to Fourier series coefficients of a periodically and \emph{symmetrically} extended sequence whereas FFT are related to Fourier series coefficients of a periodically extended sequence. 
The implicit periodicity of the FFT means that discontinuities usually occur at the boundaries of the signal. In contrast, a DCT where both boundaries are even always yields a continuous extension at the boundaries.
In particular, it is well known that any discontinuities in a function evokes higher frequencies and reduce the rate of convergence of the Fourier series, so that more sinusoids are needed to represent the function with a given accuracy. For this reason, a DCT transform uses cosine waves as its basis:
\begin{equation}
    F_k = \sum_{n=0}^{N-1}f_n \cos\Big(\frac{\pi}{N}(n+\frac{1}{2})k\Big), \quad k=0,1,\cdots,N-1
\end{equation}
where $F_k$ are the DCT coefficients of $f$. The use of cosine rather than sine basis in DCT stems from the boundary conditions implicit in the cosine functions.
In a more subtle fashion, the boundary conditions are responsible for the ``spectral compaction'' properties of DCT, since a signal's DCT representation tends to have more of its energy concentrated in a smaller number of coefficients. 
As shown in figure.~\ref{fig:my_label} for a $360^\circ-$panoramic image, by applying FFT along unwrapped $\varphi$ a circular symmetry is evoked along the horizontal axis, and by applying DCT along $z$ dimension a reflection symmetry is evoked along the vertical axis which implies smooth boundaries in both dimensions. We will show in the experiments that the use of DCT in this setting benefits overall performance of the deep learning module.


\begin{figure}[t]
    \centering
    \includegraphics[width=.6\textwidth]{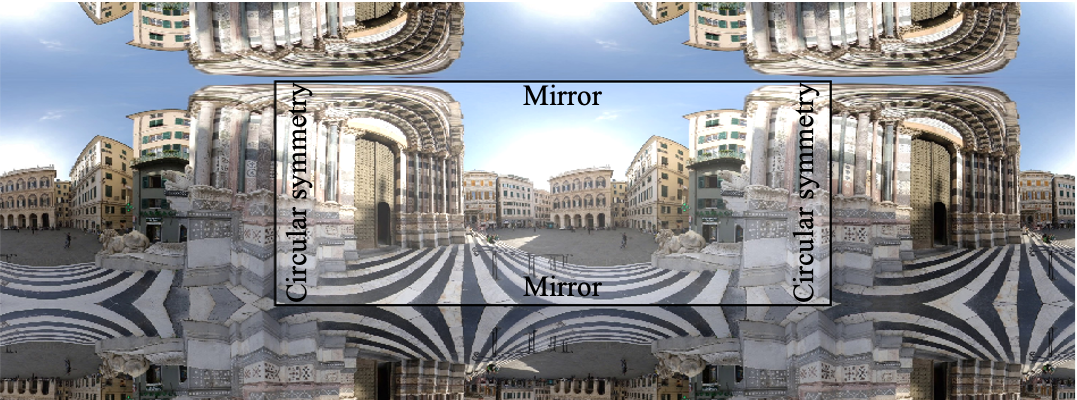}
    \caption{$360^\circ-$panoramic image with circular symmetry along the horizontal axis (unwrapped $\varphi$) and reflection symmetry along the vertical axis which are evoked by FFT and DCT, respectively.}
    \label{fig:my_label}
\end{figure}

\section{Experiments}

We begin our experiments by investigating the effect of discretizing the continuous convolution in (\ref{eq:corr2}). We then demonstrate the accuracy and effectiveness of the CCL layer in comparison with the standard convolution layer by evaluating it over a couple of well-known datasets, such as MNIST and CIFAR10. We then provide application examples for adopting CCL in designing neural networks, namely application to  3D object classification using a $360^\circ-$panoramic projection of the object on a cylindrical surface and application to change detection.       

\subsection{Discretization Error of Equivariant Layer}
In our attempt to design a group equivariant neural layer, we started by assuming the continuous group $S^1 \times \mathbb{R}$. For the purpose of implementation, it is needless to say that each function and group is discretized with a certain resolution. Therefore, as the result of discretization of the signal, the correlation kernel, and the continuous rotation group, the equivariance property does not exactly hold, i.e., we cannot exactly prove $[L_{\xi}f]\star g = L_{\xi}[f\star g]$. Furthermore, when using more than one layer, essentially a network, a moderate discretization error could propagate through the network. Therefore, it is necessary to examine the impact of discretization on the equivariance properties of the network. In other words, we want to see how this equality holds in different cases: $L_{\xi_n}\Phi(f_n) = \Phi(L_{\xi_n}f_n)$.          

\begin{figure}[h]
    \centering
    \includegraphics[width=.4\textwidth]{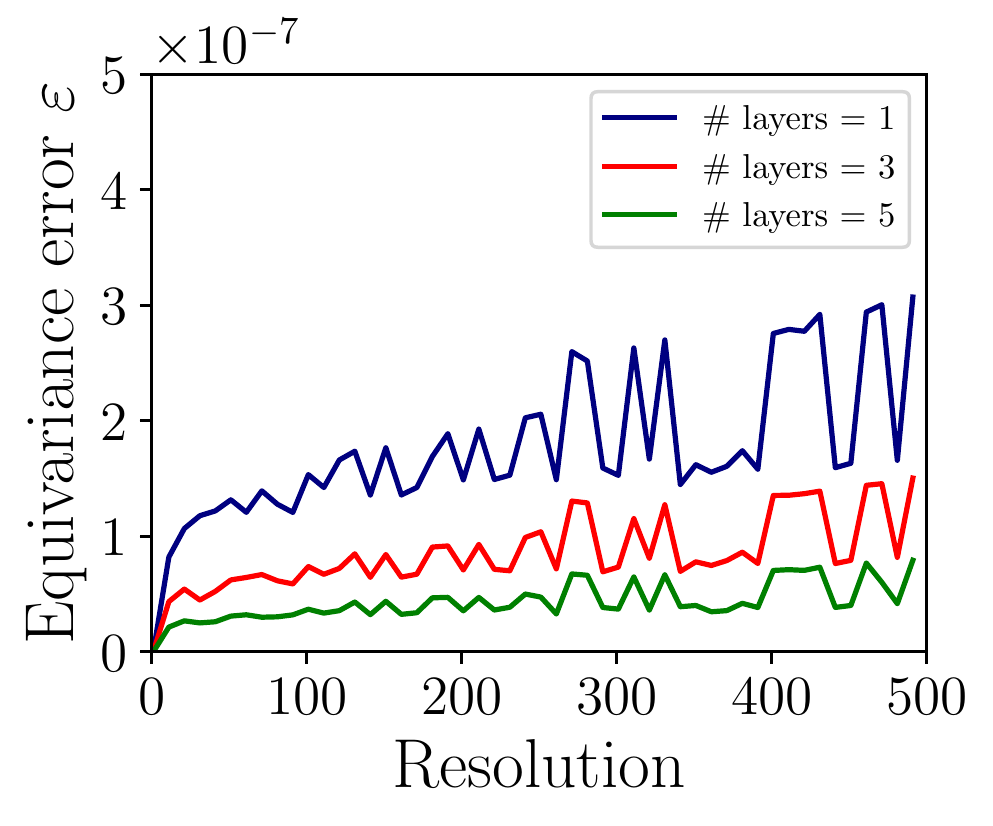}
    \caption{Equivariance approximate error as a function of resolution and number of layer.}
    \label{fig:discrete}
\end{figure}

To this aim, we test the equivariance of the CCL by sampling $N = 1000$ random input feature maps $f_n$ with $10$ input channels and for each feature map we sample a rotation $\varphi \in [0, 2\pi]$. To create and compare both side of the mentioned equivariance equality, we first shift the input feature map before passing it through the network, and then we also shift the output feature map of the intact input. We compute the discretization error as $ \varepsilon = \frac{1}{N}\sum_{n = 1}^{N} \text{std}\left(L_{\xi_n}\Phi(f_n) - \Phi(L_{\xi_n}f_n)\right)/\text{std}\left(\Phi(f_n)\right)$, where $\Phi(.)$ is a composition
of CCL correlation layers with randomly
initialized filters interleaved with relu non-linearity. In continuous case, we have perfect equivariance and $\varepsilon = 0$. However, as we can see in figure.~\ref{fig:discrete}, the approximation error is in the order of $10^-7$ and although it grows with the resolution, it decreases when we add more number of layers with relu. Furthermore, the increase rate with resolutions of the image seems to be quite low and saturating. In this figure, you can see the results for number of layers equal to 1, 3, 5.

\begin{figure}[t]
\hspace{-30pt}
    \subfloat[]{\includegraphics[height=4.1cm]{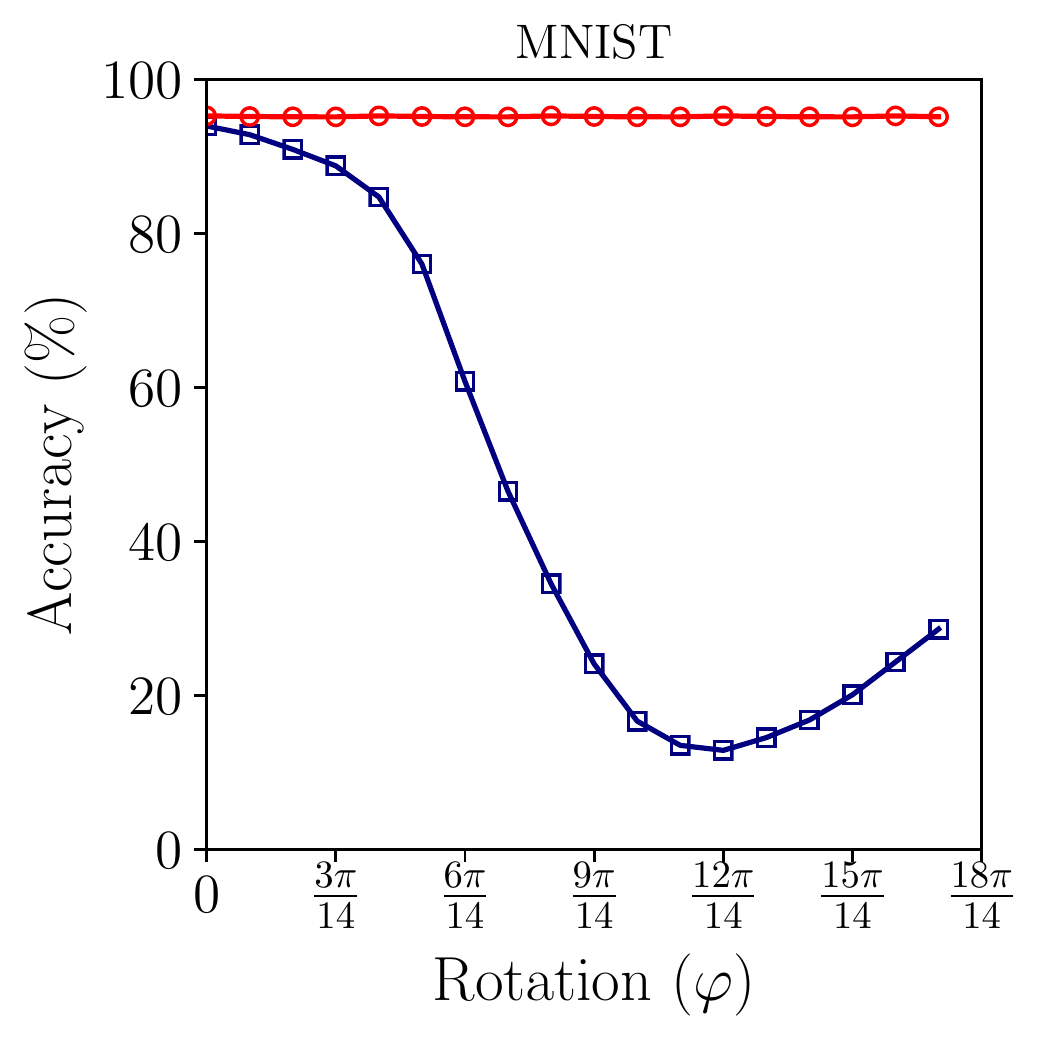}}
    \subfloat[]{\includegraphics[height=4.1cm]{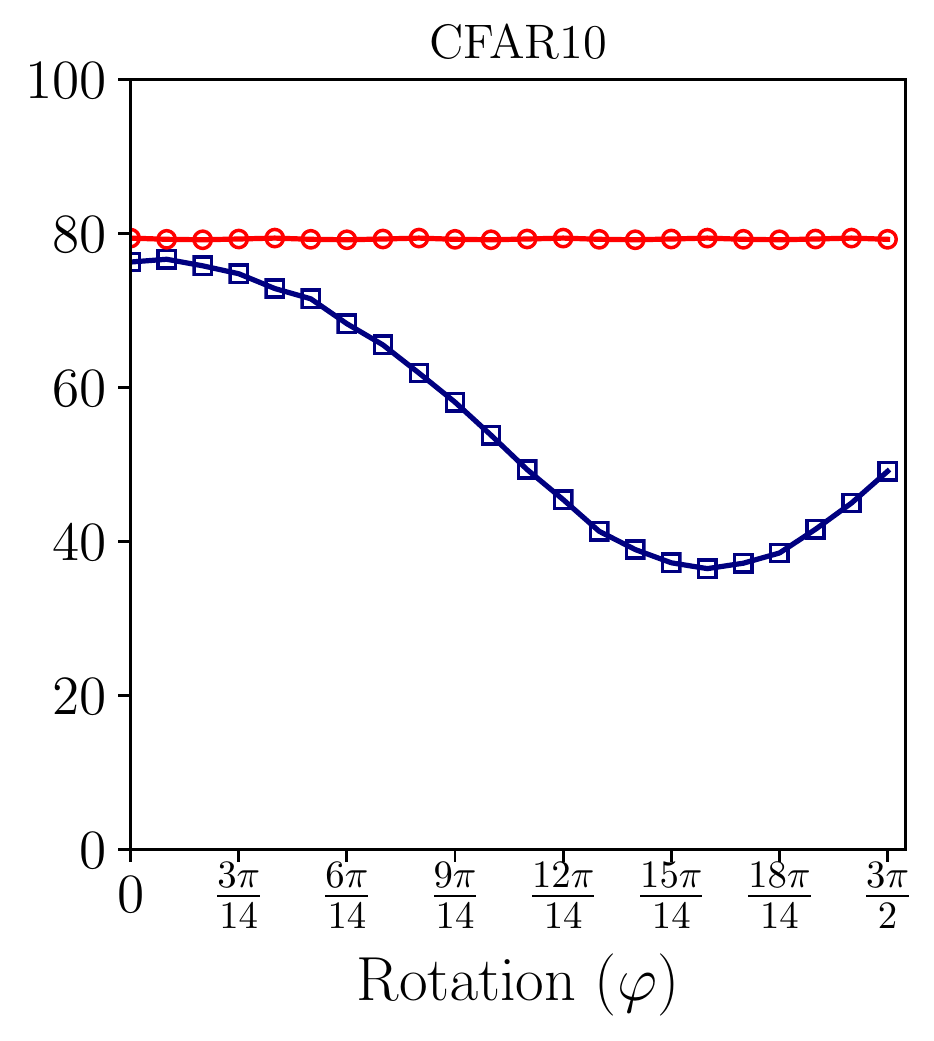}}
    \subfloat[]{\includegraphics[height=4.1cm]{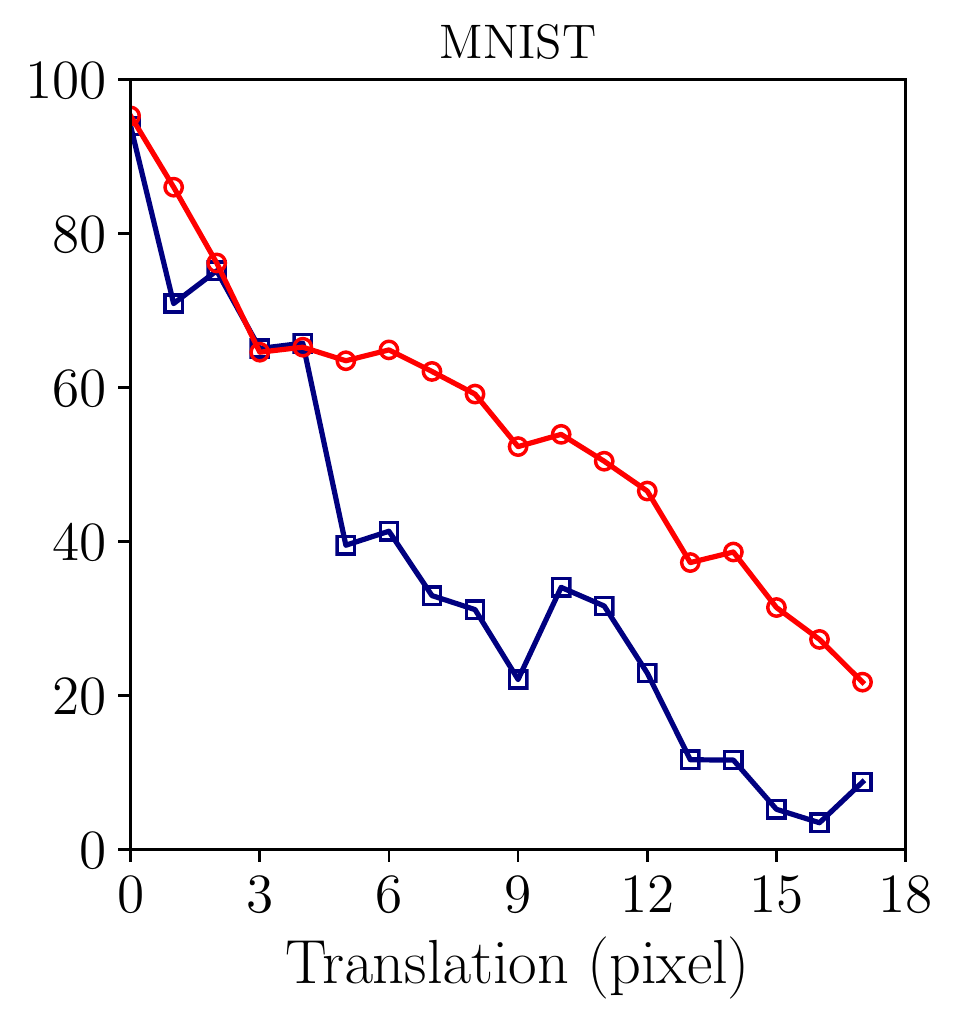}}
    \subfloat[]{\includegraphics[height=4.1cm]{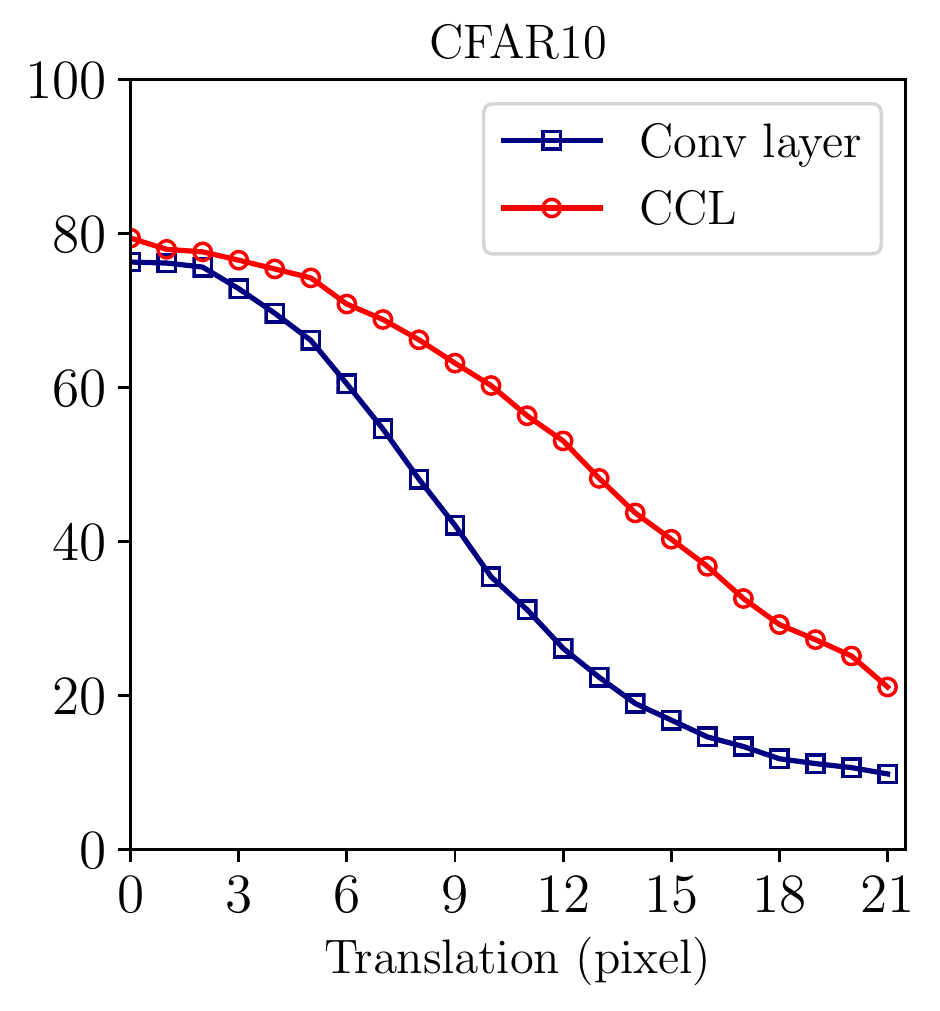}}
    \caption{Accuracy of the neural networks trained on MNIST and CIFAR10 and tested on the rolled and translated version. The two left figures show the accuracy performance of the models versus the rotation of the images around the principal axis ($z$-axis). The two right figures show the accuracy performance of the models versus translation along $z$-axis. We can see that the CCL layer is exactly equivariant with respect to $S^1$ and it demonstrate larger degree of equivariance compared to its counterpart (conv2d) with respect translation along $z$-axis.}
    \label{fig:acc}
\end{figure}

\textbf{Rolled MNIST and Rolled CIFAR10} We first evaluate the generalization performance a neural network equipped with CCL with respect to rotations of the input along $z$-axis. We propose a version of MNIST and CIFAR10 datasets called $\mathcal{R}$MNIST and $\mathcal{R}$CIFAR10, respectively, wrapped around a cylindrical surface as shown in figure.~\ref{fig:cylinder}. In this dataset, we augment the actual MNIST and CIFAR10 datasets with horizontally rolled version of the original images using random samples of  $\varphi \in [0, 2\pi]$ (see figure.~\ref{fig:cylinder}). Therefore, for a standard image size of $28\times28$, the rotation by $\nicefrac{\pi}{2}$ is equivalent to shifting the image horizontally by $\nicefrac{\pi}{2}\times\nicefrac{28}{2\pi}=7$. As the result, the images could be cut in the middle and destruct the consistency of the object in the figure, namely the digits in MNIST dataset, or the animal in CIFAR10 dataset. 
\begin{wrapfigure}[13]{r}{0.4\textwidth}
  \begin{center}
  \includegraphics[width=0.18\textwidth]{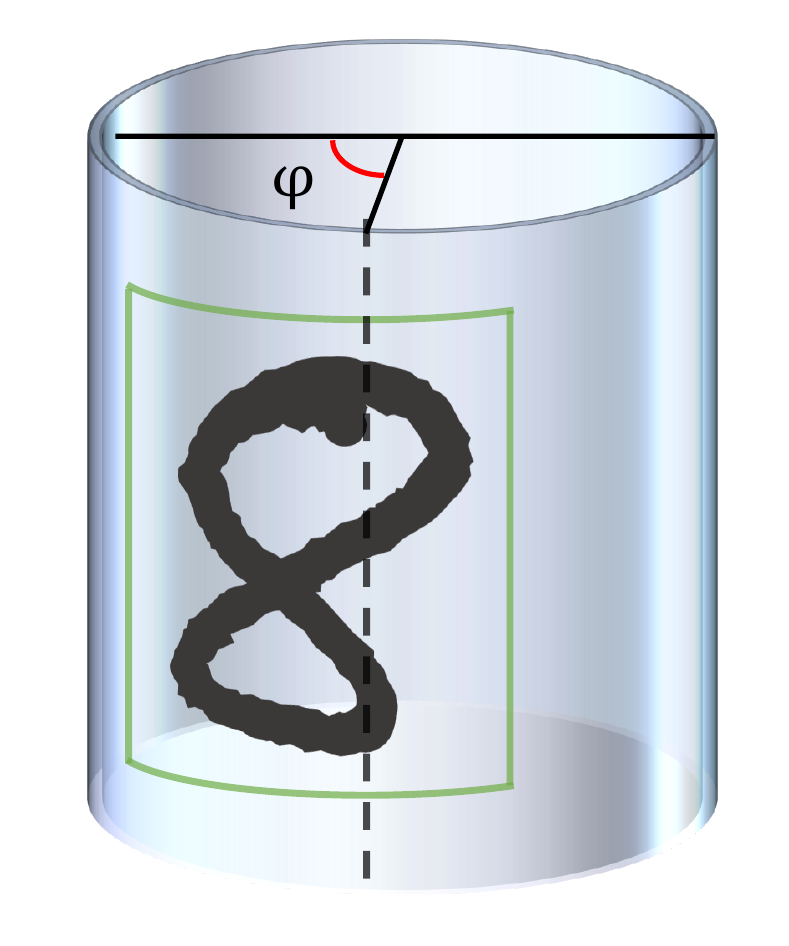}\includegraphics[width=0.19\textwidth]{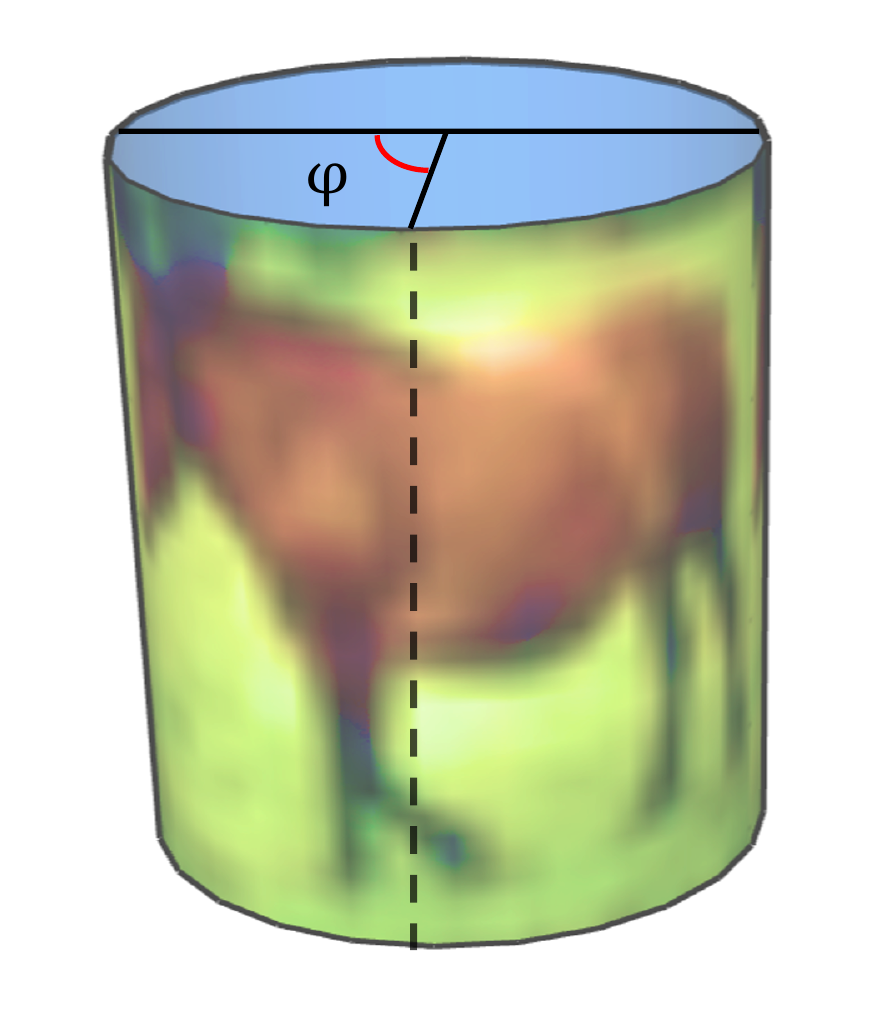}
  \end{center}
  \caption{$\mathcal{R}$MNIST and $\mathcal{R}$CIFAR10}
  \label{fig:cylinder}
\end{wrapfigure}
We perform three testing experiments using the actual datasets or their rolled versions and report the results in table.~\ref{tb:res}. As you can see, in the case of training with original MNIST and CIFAR10 and testing on the original MNIST and CIFAR10, the performance of a neural network using our CCL layer is comparable to its CNN counterpart although our network slightly outperforms. However, if we test these two neural networks, trained on the original MNIST, on the perturbed versions of the data, which are $\mathcal{R}$MNIST and $\mathcal{R}$CIFAR10, we see considerable performance drop in the case of CNN as CNNs are not equipped to deal with this degree of image translation. If we train both of these neural networks on the augmented versions of the two datasets, although the accuracy of the CNN  improves in comparison with its previous case, it is still considerably lower than a network using CCL. Furthermore, we should note that training with the augmented dataset is significantly slower than original dataset as it consists of several times more samples, i.e., for each rotation.  To see the adopted architectures refer to table.~\ref{tb:1}. To make the learned representations invariant to the rotation around the  $z$-axis a row-wise max-pooling layer
is used between the correlation and fully-connected layers (see \cite{lin2013network}).

\subsection{Invariance Analysis of Networks Built with CCL}
Here we show another set of results comparing the equivariance of neural networks adopting CCL layers and regular CNN. We adopt similar network architectures. To see the adopted architectures refer to table.~\ref{tb:1}. CCL($M$) means we used the CCL layer with output channel size of $M$. For the regular CNN we replace the CCL with regular conv2d layer and keep everything else the same.   Figure.~\ref{fig:acc} shows the accuracy of CCL neural network (red) and CNN (blue) trained on MNIST and CIFAR10 and tested on the rolled and translated version. The two left figures show the accuracy performance of the models versus different degrees of rotation of the images around the principal axis ($z$-axis). It is obvious that the CCL neural network trained only on the unperturbed data generalizes quite well in all the rotations of the test data, hence the flat red line. On the other hand, CNN performance drops as the rotation value increases up to the point the image start to roll back to its original position, hence the acute drop of the blue line. The two right figures show the accuracy performance of the models versus translation along $z$-axis. It is expected that for the finite signal along the $z$-axis the equivariance with respect to translations in this direction will be disturbed. Therefore, although the CCL layer is exactly equivariant with respect to $S^1$, it is not completely equivariant with respect to vertical translation. However networks equipped by our CCL layer demonstrate larger degree of equivariance compared to its counterpart (conv2d). This is the results of using DCT in implementing the CCL layer. Since DCT exploits the even/reflection symmetry of the images, objects remain more consistent along the upper and lower edges of the image (see figure.~\ref{fig:my_label}).

\subsection{Application to 3D object classification}
\begin{figure}[h]
\begin{center}
    \subfloat[]{\includegraphics[width=.25\textwidth]{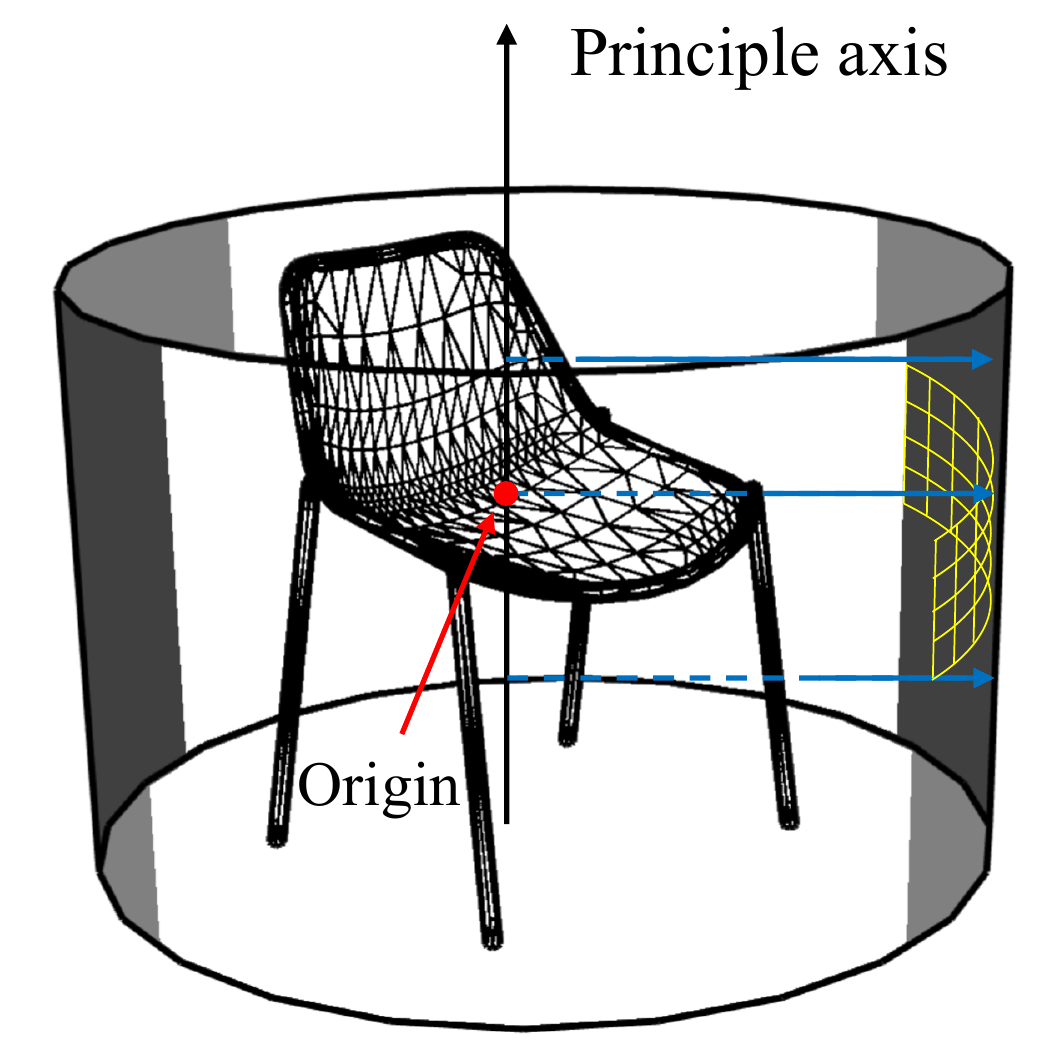}}\hspace{.5cm}
    \subfloat[]{\abovebaseline[10pt]{\includegraphics[width=.35\textwidth]{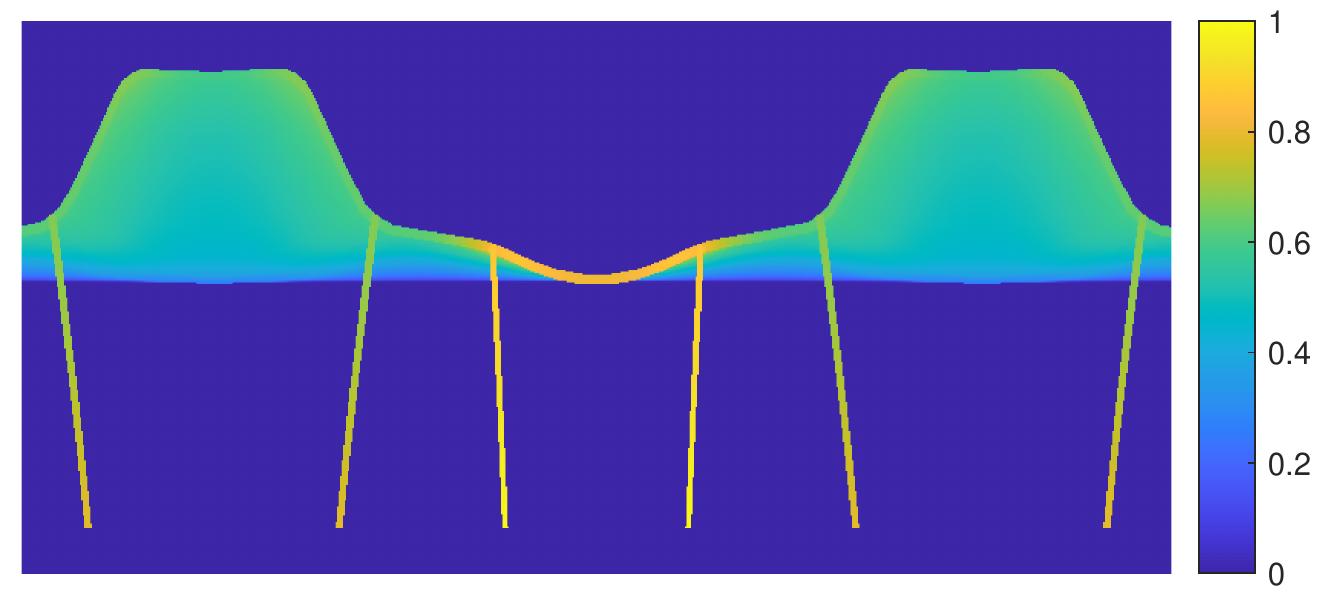}}}
\end{center}
    \caption{$360^\circ-$panoramic view of a 3D object: (a) We cast a ray from each pixel on the cylinder surface to the 3D object and measure the depth. (b) Depth measurements form a $360^\circ-$panoramic image.}
    \label{fig:raycast}
\end{figure}

We evaluate the effectiveness of our model in image-based 3D object classification, that is using multiple views (image) of a 3D object to perform the classification task. Here, we adopt a continuous panoramic views of a 3D object which describe the position of the object’s surface with respect to a cylindrical surface in 3D space \citep{yavartanoo2018spnet, sfikas2018ensemble,shi2015deeppano, papadakis2010panorama}. 
We use ShapeNetCore, a subset of the full ShapeNet dataset, \citep{chang2015shapenet}, with 51,300 unique 3D models that covers 55 common object categories. This data set consists of a regular dataset of consistently aligned 3D models and another dataset in which models that are perturbed by random rotations in 3D. For our study, we aim to use $360^\circ-$panoramic images with rotations are along a specific axis (e.g. $z$-axis). We construct this mentioned dataset by rolling the $360^\circ-$panoramic images in figure.~\ref{fig:raycast}(b) using random samples of $\varphi \in [0,2\pi]$. According to table.~\ref{tb:res}, the classification accuracy of our model is higher than that of CNN for this dataset. See table.~\ref{tb:1} for information regarding the architecture.

\begin{table}[h]
\centering
\setlength\tabcolsep{5pt}
\caption{Accuracy results for network using CCL and Conv2d layers.}
\begin{tabular}[t]{lccccccc}
\toprule

Train set
& MNIST 
& MNIST 
& $\mathcal{R}$MNIST 
& CIFAR10 
& CIFAR10 
& $\mathcal{R}$CIFAR10 
& 3D Object
\\
Test set 
& MNIST 
& $\mathcal{R}$MNIST 
& $\mathcal{R}$MNIST 
& CIFAR10 
& $\mathcal{R}$CIFAR10 
& $\mathcal{R}$CIFAR10 
& 3D Object
\\
\midrule
Conv2d 
& 93.96 
& 16.81 
& 46.64
& 76.29 
& 49.11  
& 66.07
& 68.08
\\
\textbf{Ours} 
& \textbf{95.27}  
& \textbf{95.15} 
& \textbf{95.35}
& \textbf{79.40}
& \textbf{79.24} 
& \textbf{79.02} 
& \textbf{82.14}
\\
\bottomrule
\end{tabular}
\label{tb:res}
\end{table}%

\begin{table}[!ht]
\caption{Network architectures. CCL$($c$_\textsc{out})$: c$_\textsc{out}$ implies number of output channels. $\text{FC}(l_\textsc{in}, l_\textsc{out})$: $l_\textsc{in}$ and $l_\textsc{in}$ imply input and output features dimensions, respectively. $\text{MaxPool}(k, s)$: $k$ and $s$ imply kernel and stride sizes, respectively.
$\text{AvgPool}(k)$: $k$ implies kernel sizes. $\text{BN}$ denotes batch normalization. Note that the global average pooling makes the network invariant to the input roll.
}
\centering
{\small
\setlength\tabcolsep{9pt}
\begin{tabular}{c|c|c|c}
\toprule
\backslashbox{Layer}{Dataset} &
MNIST
& CIFAR10 &
Panoramic ShapeNet
\\
\midrule
Input
& $f\in \mathbb{R}^{1\times 28\times28}$
& $f\in \mathbb{R}^{3\times 32\times32}$
& $f\in \mathbb{R}^{1\times 48\times100}$
\\ 
\midrule
1
& \text{CCL(8), ReLU}
& \text{CCL(128), ReLU}
& \text{CCL(64), BN, ReLU}
\\
2
& \text{CCL(8), ReLU}
& \text{CCL(128), ReLU}
& \text{MaxPool(2, 2)}
\\
3
& \text{MaxPool(2, 2)}
& \text{MaxPool(2, 2)}
& \text{CCL(64), BN, ReLU}
\\
4
& \text{CCL(8), ReLU}
& \text{CCL(128), ReLU}
& \text{CCL(128), BN, ReLU}
\\
5
& \text{CCL(8), ReLU}
& \text{CCL(256), ReLU}
& \text{MaxPool(2, 2)}
\\
5
& \text{MaxPool(2, 2)}
& \text{MaxPool(2, 2)}
& \text{CCL(256), BN, ReLU}
\\
6
& \text{CCL(10), ReLU}
& \text{AvgPool(8)}
& \text{AvgPool(300)}
\\
7
& \text{AvgPool(7), Softmax}
& \text{FC(256, 120), ReLU}
& \text{FC(256, 100), ReLU}
\\
8
& 
& \text{FC(120, 84), ReLU}
& \text{FC(100, 55), Softmax}
\\
9
&
& \text{FC(84, 10), Softmax}
&
\\
\bottomrule
\multicolumn{4}{l}{\small For regular CNN, the CCL layers are replaced with Conv2d layers.}
\end{tabular}
}
\label{tb:1}
\end{table}

\subsection{Application to change detection}

\begin{figure}
    \centering
    \includegraphics[width=\textwidth]{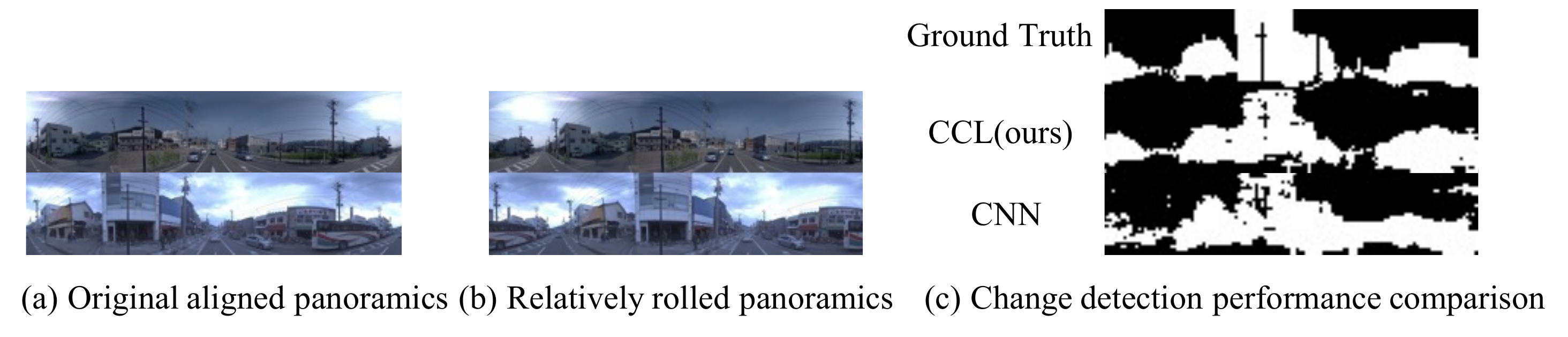}
    \caption{Change detection performance comparison for Tsunami dataset.}
    \label{fig:change}
\end{figure}
We further evaluated CCL in image-based detection of temporal scene changes. We used the Tsunami change detection dataset \citep{sakurada2015change} which consists of one hundred aligned panoramic image pairs of scenes in tsunami-damaged areas of Japan and the corresponding ground-truth mask of scene changes.
We generated a randomly rolled version of this dataset in which the first and second views are relatively rolled with respect to each other. 
We designed a neural network that is equivariant to the roll of the first view and invariant to the roll of the second view. 
The network consists of two CCL(50)-MaxPool-ReLU layers followed by two CCL(50)-Transpose-ReLU and a final CCL(1)-Sigmoid layer which regresses the change mask. The second layer is followed by a global AvgPool across feature maps of the second view (to make it invariant) and is then summed with feature maps of the first view. CCL-Transpose layer is implemented by up-sampling feature maps with zeros of appropriate stride. 
We tested our model against CNN version of the same architecture and achieved F1-score of 88.91\% compared to 76.46\% on the test set. As depicted in figure.~\ref{fig:change} our model generalizes better to the test set due to its equivariance to $S^1$.

\section{Discussion and Conclusion}
We  have proposed a Circular-symmetric Correlation Layer (CCL) based on the formalism of roto-translation equivariant correlation on the continuous group $S^1 \times \mathbb{R}$, and implement it efficiently using the well-known FFT and DCT algorithm. Our numerical results demonstrates the effectiveness and accuracy obtained from adopting CCL layer. A neural network equipped with CCL generalizes across rotations around the principle axis, and outperforms its CNN counterpart on competitive 3D Model Recognition. It is worthwhile to note that the achieved gain is not at the expense of increasing the number of parameters (by zero- or input-padding of the input data), or at the expense of data augmentation and hence longer training time and sample complexity. It is merely due to the intrinsic property the CCL layer in mimicking the circular symmetry and reflection symmetry in the data.


\newpage
{
\small
\bibliographystyle{abbrvnat}
\bibliography{references.bib}
}

\newpage


\newpage
\pagenumbering{arabic} 
\setcounter{page}{1}
{\centering
{\LARGE\bf Supplementary Material\par}
}
\setcounter{section}{0}
\renewcommand{\thesection}{\Alph{section}}
\section{Background on Group Theory}
The formalism used in the paper is based on various concepts in group theory and abstract algebra. In this part, we aim to provide this key concepts, and their corresponding notations and definitions.

\paragraph{Symmetry:} A symmetry is a set of transformations applied to a structure. The transformations should preserve the properties of the structure. Generally it is also presumed that the transformations must be invertible, i.e. for each transformation there is another transformation, called its inverse, which reverses its effect.
\par Symmetry is thus can be stated mathematically as an operator acting on an object, where the defining feature is that the object remains unchanged. In other words, the object is invariant under the symmetry transformation. Symmetries are modeled by \textbf{Groups}.

\paragraph{Group:} Let $\mathcal{G}$ be a non-empty set with a binary operation defined as $\circ: \mathcal{G}\times\mathcal{G}\mapsto\mathcal{G}$. We call the pair $(\mathcal{G}; \circ)$ a group if it has the following properties: 
\begin{itemize}
    \item[] \textbf{(Closure)}: $\mathcal{G}$ is closed under its binary operation,
    \item[] \textbf{(Associativity axiom)}: the group operation is associative –i.e., $(g_1\circ g_2)\circ g_3 = g_1\circ (g_2\circ g_3)$ for $g_1, g_2, g_3 \in \mathcal{G}$,
    \item[] \textbf{(Identity axiom)}: there exists an identity $e \in \mathcal{G}$ such that $g\circ e = e\circ g = g$ for all $g \in \mathcal{G}$, 
    \item[] \textbf{(Inverse axiom)}: every element $g \in \mathcal{G}$ has an inverse $g^{-1} \in \mathcal{G}$, such that $g\circ g^{-1} = g^{-1}\circ g = e$.
\end{itemize}
\paragraph{Subgroup:}  A non-empty subset $\mathcal{H}$	of $\mathcal{G}$ is called a subgroup, if $\mathcal{H}$ is a group equipped with the same binary operation of as in $\mathcal{G}$. We  show this as $\mathcal{H}\leq\mathcal{G}$. $\mathcal{H}$ is called a proper subgroup of  if $\mathcal{H}\not=\mathcal{G}$ and we show it as $\mathcal{H}<\mathcal{G}$.

\paragraph{Group order:} 
The number of elements in a group $\mathcal{G}$ is called the \textbf{order} of $\mathcal{G}$ and is denoted $|\mathcal{G}|$. $\mathcal{G}$ is called a finite group if $|\mathcal{G}|<\infty$ and infinite otherwise.
 
 \paragraph{Group action:} We are interested on the way a group “acts” on the input
and output of a deep network. Function $\gamma:\mathcal{G}\times\mathbb{X}\to \mathbb{X}$ is the left action of group $\mathcal{G}$ on $\mathbf{x}$ iff I) 
$\gamma(e,\mathbf{x}) = \mathbf{x}$ and; II)$\gamma(g_1,\gamma(g_2,\mathbf{x})) = \gamma(g_1g_2,\mathbf{x})$.

\paragraph{Faithful $\mathcal{G}$-action:} $\mathcal{G}$-action is faithful iff two groups are isomorphic $\mathcal{G}\cong \mathcal{G}_{\mathbb{N}}$.

\paragraph{Normal subgroup:} For $\mathcal{H}$, a subgroup of a group $\mathcal{G}$, the similarity transformation of $\mathcal{H}$ by a fixed element $g$ in $\mathcal{G}$ not in $\mathcal{H}$ always gives a subgroup. If 
\[
gHg^{-1} = H
\]
for every element $g$ in $\mathcal{G}$, then $\mathcal{H}$ is said to be a normal subgroup of $\mathcal{G}$, written $\mathcal{H} \lhd \mathcal{G}$. Normal subgroups are also known as invariant subgroups or self-conjugate subgroup.

\paragraph{Homogeneous Space and Transitivity:} Transitivity is the property that taking any $x_0 \in \mathcal{X}$, any other $x\in \mathcal{X}$ can be reached by the action of some $g \in \mathcal{G}$, i.e., $x= g(x_0)$. If the action of $\mathcal{G}$ on $\mathcal{X}$ is \gls{transitive},
we say that $X$ is a \gls{homogeneous space} of $\mathcal{G}$.

\paragraph{Homomorphism:} Let $\mathcal{G}$ with binary operation $\circ$ and $\mathcal{H}$ with binary operation $\star$ be groups. The map $ \Phi: \mathcal{G} \rightarrow \mathcal{H}$ is called a homomorphism from $(\mathcal{G},\circ)$ to $(\mathcal{H},\star)$, if for all $g_1, g_2 \in \mathcal{G}$ we have:
\[
\Phi(g_1\circ g_2) = \Phi(g_1)\star \Phi(g_2).
\]

\par A homomorphism $ \Phi: \mathcal{G} \rightarrow \mathcal{H}$ is called

\begin{itemize}
    \item[] monomorphism if the map $\Phi$ is injective,
    \item[] epimorphism if the map $\Phi$  is surjective,
    \item[] isomorphism if the map $\Phi$ is bijective,
    \item[] endomorphism if $\mathcal{G} = \mathcal{H}$,
    \item[] automorphism if $\mathcal{G} = \mathcal{H}$ and the map $\Phi$ is bijective.
\end{itemize}

\paragraph{Isomorphic:} Two groups are isomorphic $\mathcal{G} \cong \mathcal{H}$ if there exists a bijection $ \Phi: \mathcal{G} \rightarrow \mathcal{H}$ between them.

\section{Code, Datasets, and Experimental Settings}
In this section, we explain the details of dataset preparation, code, and experimental settings.
\subsection{Datasets}
\textbf{Rolled MNIST and Rolled CIFAR10} We first evaluate the generalization performance a neural network equipped with CCL with respect to rotations of the input along $z$-axis. We propose a version of MNIST and CIFAR10 datasets called $\mathcal{R}$MNIST and $\mathcal{R}$CIFAR10, respectively, wrapped around a cylindrical surface as shown in figure.~\ref{fig:cylinder}. 

\begin{wrapfigure}[30]{r}{0.4\textwidth}
  \begin{center}
  \includegraphics[width=0.35\textwidth]{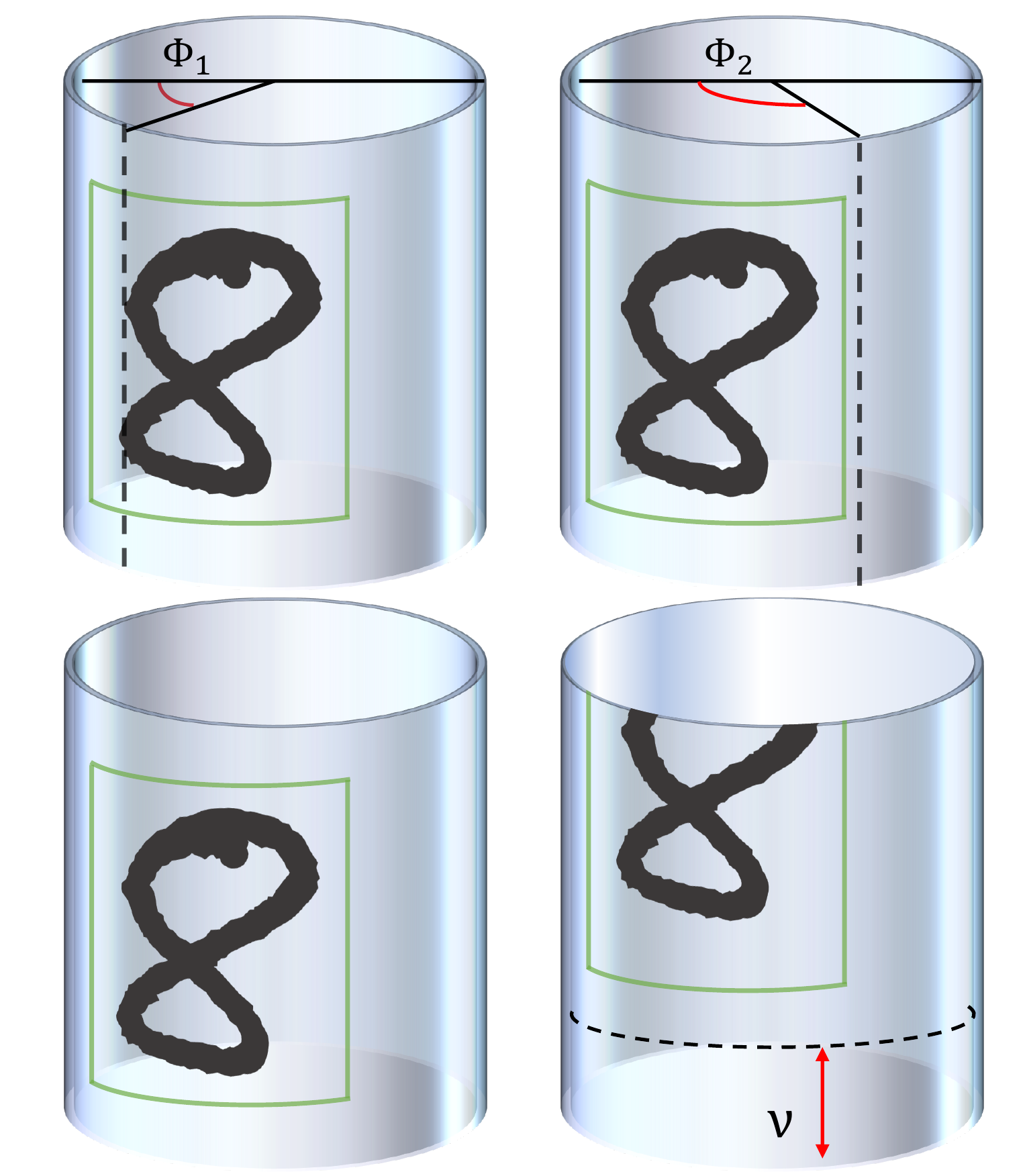}
  \end{center}
  \caption{$\mathcal{R}$MNIST. For the dataset to be representative of all the define transformations mentioned in the paper, namely, rotation around the $z$-axis and translation along the  $z$-axis, we randomly generated the discretised rolls ($\varphi_i\in [0,2\pi]$ with step size of $1/28$).\textbf{bottom-left}: original cylindrical image. \textbf{bottom-right}: The image is translated up by $\nu$. \textbf{up-left}: panoramic is rolled by $\varphi_1$ (the image is cut in the middle). \textbf{up-right}: panoramic is rolled by $\varphi_2$.  }
  \label{fig:suppcy}
\end{wrapfigure}
In this dataset, we augment the actual MNIST and CIFAR10 datasets with horizontally rolled version of the original images using random samples of  $\varphi \in [0, 2\pi]$ (see figure.~\ref{fig:cylinder}). Therefore, for a standard image size of $28\times28$, the rotation by $\nicefrac{\pi}{2}$ is equivalent to shifting the image horizontally by $\nicefrac{\pi}{2}\times\nicefrac{28}{2\pi}=7$. As the result, the images could be cut in the middle and destruct the consistency of the object in the figure, namely the digits in MNIST dataset, or the animal in CIFAR10 dataset. 

\par 
For the dataset to be representative of all the define transformations mentioned in the paper, namely, rotation around the $z$-axis and translation along the  $z$-axis, we randomly generated the discretised rolls ($\varphi_i\in [0,2\pi]$ with step size of $1/28$). As illustrated in Figure.~\ref{fig:suppcy}(\textbf{bottom-left}: original cylindrical image. \textbf{bottom-right}: The image is translated up by $\nu$. \textbf{up-left}: panoramic is rolled by $\varphi_1$ (the image is cut in the middle). \textbf{up-right}: panoramic is rolled by $\varphi_2$) this transformations will disturb the perception neural network.
\par
\textbf{Panoramic Change Detection Dataset:} TSUNAMI dataset,  \cite{sakurada2015change}, consists of one hundred panoramic image pairs of scenes in tsunami-damaged areas of Japan. The size of these images is $224 \times 1024$ pixels. For each image, they hand-labeled the ground truth of scene changes. It is given in the form of binary image of the same size as the input pair of images. The binary value at each pixel is indicative  of the change that occurred at the corresponding scene point on the paired images. The scene changes are defined to be detected as 2D changes of surfaces of objects (e.g., changes of the advertising board) and 3D, structural changes (e.g., emergence/vanishing of buildings and cars). The changes due to differences in illumination and photographing condition and those of the sky and the ground are excluded, such as changes due to specular reflection on building windows and changes of cloud and signs on the road surface. For the ground-truth, all image pairs have ground truths of temporal scene changes, which are manually obtained by the authors in \cite{sakurada2015change}.

\begin{figure}[!h]
    \centering
    \includegraphics[width=0.9\textwidth]{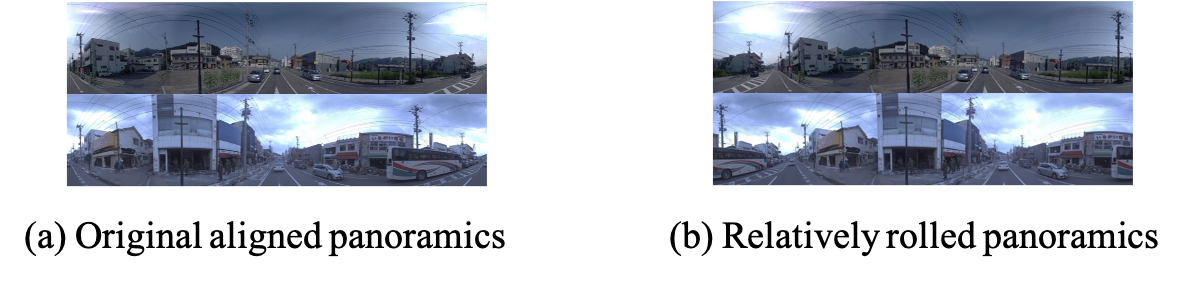}
    \caption{TSUNAMI dataset. (a) The original dataset with the two panoramic images aligned at the time of capture. (b) The dataset we built based on the original one to roll/shift one of the panoramic image relative to the other.}
    \label{fig:suppDetec}
\end{figure}

\par The original dataset is captured in a way that the two panoramic images to be aligned completely as it can be seen in Figure.~\ref{fig:suppDetec}. In order to make the dataset more challenging and realistic, and  to demonstrate the power of using CCL on this dataset, we built a variation of the dataset based on the original one by rolling/shifting one of the panoramic image relative to the other. In this way the network needs to be invariant to these rolls in the images in order to perform the task.   

\subsection{Network architecture for change detection problem}

\begin{table}[!h]
\caption{Network architectures. CCL$($c$_\textsc{out})$: c$_\textsc{out}$ implies number of output channels. $\text{FC}(l_\textsc{in}, l_\textsc{out})$: $l_\textsc{in}$ and $l_\textsc{in}$ imply input and output features dimensions, respectively. $\text{MaxPool}(k, s)$: $k$ and $s$ imply kernel and stride sizes, respectively.
$\text{AvgPool}(k)$: $k$ implies kernel sizes. $\text{BN}$ denotes batch normalization. $\text{UpSample}(n_1,n_2)$ up-samples 2D dimensions of the input with stride $n_1$ and $n_2$. Note that the global average pooling makes the network invariant to the input roll.}
\centering
{\small
\setlength\tabcolsep{3pt}
\begin{tabular}{c|c}
\toprule
\backslashbox{Layer}{Dataset} &
Panoramic Change Detection TSUNAMI Dataset
\\
\midrule
Input
& $f\in \mathbb{R}^{1\times 28\times28}$

\\ 
\midrule
1
& \text{CCL(50), ReLU}

\\
2
& \text{MaxPool(2, 4)}

\\
3
& \text{CCL(50), ReLU}

\\
4
& \text{MaxPool(2, 4)}

\\
5
& \text{UpSample(2,4)}

\\
6
& \text{CCL(1), Sigmoid}

\\

\bottomrule
\multicolumn{2}{l}{\small For regular CNN, the CCL layers are replaced with Conv2d layers.}
\end{tabular}
}
\label{tb:supp}
\end{table}

\subsection{Code and Computational Resources}
We implemented CCL and performed our experiemnts in PyTorch v1.8 \citep{paszke2017automatic} and used the Adam optimizer \citep{kingma2014adam} with learning rate of $0.001$. We initialized all parameters randomly. Each CCL layer has $C_{out}\times C_{in} \times K_1 \times K_2$ parameters similar to that of a CNN layer. The
time complexity for training a CCL layer is $O(C_{in}\times C_{out}\times N \times \log(N))$ per data sample and epoch, where $N$ is the dimension of input data, and the space complexity is $O(b\times C_{in} \times C_{out} \times N)$ where $b$ is the batch size.
We learned and tested all the models on an Intel Core i9 CPU@3.6GHz with 64 GB of RAM and Nvidia GeForce RTX 2080 with 11 GB of RAM . Per-epoch training time varied from $30$ms in smaller datasets to $1$s in larger experiments and $25$ epochs sufficed for all experiments.

{
\small
}



\end{document}